\documentclass[runningheads]{llncs}
\usepackage{graphicx}
\usepackage{amsmath,amssymb,mathrsfs,amsfonts,dsfont,verbatim} 
\usepackage[margin=1.5in]{geometry}
\begin{document}
\title{PAC Learning Guarantees Under Covariate Shift}
\author{Artidoro Pagnoni\inst{1,2}\and
Stefan Gramatovici\inst{1,2} \and
Samuel Liu\inst{2}}
\authorrunning{Pagnoni, Gramatovici, and Liu}
\institute{Microsoft AI \and
Harvard University\\
\email{\{artidoro.pagnoni,stefan.gramatovici\}@microsoft.com}\\
\email{samuelzliu@college.harvard.edu}}

\maketitle              % typeset the header of the contribution
\begin{abstract}
We consider the Domain Adaptation problem, also known as the covariate shift problem, where the distributions that generate the training and test data differ while retaining the same labeling function. This problem occurs across a large range of practical applications, and is related to the more general challenge of transfer learning. Most recent work on the topic focuses on  optimization techniques that are specific to an algorithm or practical use case rather than a more general approach. The sparse literature attempting to provide general bounds seems to suggest that efficient learning even under strong assumptions is not possible for covariate shift. Our main contribution is to recontextualize these results by showing that any Probably Approximately Correct (PAC) learnable concept class is still PAC learnable under covariate shift conditions with only a polynomial increase in the number of training samples. This approach essentially demonstrates that the Domain Adaptation learning problem is as hard as the underlying PAC learning problem, provided some conditions over the training and test distributions. We also present bounds for the rejection sampling algorithm, justifying it as a solution to the Domain Adaptation problem in certain scenarios.

\end{abstract}
\section{Introduction}

\subsection{Problem Setting}
The standard machine learning problem formulation assumes that training and test data are generated by the same underlying process. Intuitively, it is only possible to learn that which has been already experienced; the objective of training is to expose the learning model to data that is similar to what it will be expected to perform on. However, the assumption that the training and test data come from the same distribution is restrictive. There are many situations where test and training distributions differ, including drift of the generative process over time or unavailability of data from the target distribution. Dealing with this relaxation of the standard ML assumption is known as Domain Adaptation (DA). Consider, for example, the task of predicting health outcomes in the general population, while only having access to data from university affiliates across the United States. The training distribution will be a biased sample, as individuals affiliated with a university will likely be younger than the general population and have access to better health care on average. Blindly applying a learning algorithm to the available training data might lead to inaccurate and unrepresentative predictions for the general population. As such, the challenge of DA is to generalize learning from one domain to another, and thus has seen many practical applications from sentiment analysis \cite{DBLP:conf/aaai/AdelW15} to spam filtering \cite{DBLP:journals/jmlr/BickelBS09} to computational biology \cite{DBLP:conf/ismb/BorgwardtGRKSS06}.

Similarly, the standard Probably Approximately Correct (PAC) learning paradigm \cite{DBLP:journals/cacm/Valiant84} assumes that the training and test (target) data are generated from the same distributions. The PAC learning paradigm is one where a learner relies on a set of labeled examples (training set) from which it produces a hypothesis. The objective of the learner is to construct the hypothesis that most closely matches the original concept that generated the labeled training examples. PAC learning can be used to analyze the efficiency and bounds of learning algorithms in general rather than any one specific algorithm. If we wish to say anything about our general and theoretical ability to solve the Domain Adaptation challenge (as opposed to ``domain'' and algorithm specific results), then studying a modified version of PAC learning is a natural avenue to pursue. In particular, we will investigate PAC learning in the context of covariate shift \cite{DBLP:books/daglib/0001548}, a sub-problem of Domain Adaptation where the test and target distributions differ but share a labeling function. Our work contributes to the existing literature, by showing that it is possible to efficiently PAC learn under a covariate shift with the standard assumption that the support of the target distribution is contained in the support of the training distribution. In particular, we will prove that if a concept class is PAC learnable without covariate shift (implying that the source and target distributions are the same), then complicating the problem by introducing covariate shift will still allow for polynomial time PAC learning. This essentially provides an upper-bound on the difficulty of the domain adaptation problem. We will then show how this upper-bound can be improved on in the discrete and finite variance scenario using rejection sampling. 

The paper is outlined in the following manner. In section 2, we introduce important concepts and definitions. We then prove our results in section 3 and 4. Finally, we will discuss their implications and conclude in section 5. 

\subsection{Related Work}
Much of the work on Domain Adaptation has focused on practical improvements for specific machine learning algorithms such as deep neural nets \cite{DBLP:conf/icml/IoffeS15}, linear classifiers \cite{DBLP:conf/icml/GermainHLM13}, or regression \cite{DBLP:conf/alt/CortesM11}. There is also literature on techniques to convert one domain to the other by re-weighting distributions \cite{weight,DBLP:journals/jip/TsuboiKHBS09} or rejection sampling \cite{DBLP:journals/sigpro/MartinoM10}, and related work on dealing with different application domains \cite{DBLP:conf/aaai/AdelW15,DBLP:journals/jmlr/BickelBS09,DBLP:conf/ismb/BorgwardtGRKSS06}.

Attempts to prove bounds for the Domain Adaptation problem tend to rely on assumptions that limit the scope of the results. For instance, there are bounds on specific learning algorithms such as linear classifiers in Germain et al. \cite{DBLP:conf/icml/GermainHLM13}, or nearest neighbors in Ben-David et al. \cite{DBLP:journals/amai/Ben-DavidU14}, and bounds on generalization error estimators in Sugiyama et al. \cite{Sugiyama05generalizationerror}.

More general work on Domain Adaptation bounds has been conducted by Bartlett who introduced the related problem of learning with drifting distributions \cite{DBLP:conf/colt/Bartlett92} and proved PAC guarantees. In this setting, the joint distribution over the input and labels is allowed to drift over time under the assumption that the $L1$ distance between the distributions in two consecutive time steps is bounded. Recently, Mohri et al. \cite{DBLP:conf/alt/MohriM12} demonstrated tighter bounds based on a different distance metric for distributions.

However, the most prolific author on this topic is Ben-David et al. \cite{Ben-DavidBCKPV10,DBLP:journals/jmlr/Ben-DavidLLP10,Ben-DavidU12}. We will discuss \cite{Ben-DavidU12} in detail, which posits that DA is generally hard (not polynomial in the dimension of the input space), and recontextualize its main result.

\section{Preliminaries}
\subsection{Definitions}
The Probably Approximately Correct (PAC) learning model introduced by Valiant \cite{DBLP:journals/cacm/Valiant84} provides a set of abstractions to study learning tasks. This approach takes into account both accuracy of predictions and confidence of such accuracy. We begin by defining PAC learning.

Consider the following learning setup. Let $\mathcal{X}$ denote an input space and $\mathcal{Y}$ the label space. We let $\mathcal{C}$ denote a concept class which is a set of functions $c$ that map elements in the input space to elements in the label space $c : \mathcal{X} \rightarrow \mathcal{Y}$. Let $n$ be the dimension of the input space $\mathcal{X}$, and $size(c)$ the size of the smallest representation of $c$. The objective of a general learning task is to predict with high accuracy the label of elements of $\mathcal{X}$ given some training samples. The training samples are labeled examples drawn from $\mathcal{X}$ according to some unknown distribution $\mathcal{D}$. To model this process, it is customary in learning theory to consider an oracle $EX(c, \mathcal{D})$ which outputs labeled examples $\langle x, c(x)\rangle$ one at a time such that $x\sim P_{\mathcal{D}}$. After training, we say that a learning algorithm produces a hypothesis $h$, a map $h: \mathcal{X}\rightarrow \mathcal{Y}$ from input to label space. It is natural to measure the accuracy of such a hypothesis using $error_{\mathcal{D}}(h) = P_{x\sim P_{\mathcal{D}}}(h(x)\neq c(x))$. In learning scenarios, what matters is that the accuracy is high for common samples which will be built in the PAC model through the confidence parameter.

\begin{definition}[Probably Approximate Correct model]
A concept class $\mathcal{C}$ over $\mathcal{X}$ is PAC learnable if there exists an algorithm $\mathcal{A}$ such that for every concept $c\in \mathcal{C}$, for every distribution $\mathcal{D}$ on $\mathcal{X}$, and for all $\epsilon, \delta > 0$ if $\mathcal{A}$ is given access to $EX(c, \mathcal{D})$ and inputs $\epsilon$ and $\delta$, then with probability at least $1-\delta$, $\mathcal{A}$ outputs a hypothesis $h \in \mathcal{C}$ satisfying $error_{\mathcal{D}}(h)\leq\epsilon$. This probability is taken over random examples drawn by calls to $EX(c, \mathcal{D})$

If $\mathcal{A}$ runs in time polynomial in $1/\epsilon$, $1/\delta$, $n$ and $size(c)$ we say that $\mathcal{C}$ is efficiently PAC learnable.
\end{definition}

\noindent Note that $\epsilon$ and $\delta$ will refer to the error parameter and the confidence parameter respectively.

In this paper we focus on Domain Adaptation (DA) learning. DA learning differs from PAC learning theory in only two key assumptions. In DA learning we have a training distribution that is different from the distribution on which the algorithm will be tested. In data science and statistics this problem is also known as covariate shift. We call $\mathcal{D}_S$ the source distribution on the training set, and $\mathcal{D}_T$ the target distribution on the test set. The second difference is that in DA we allow the learner to have access to an oracle $EX(\mathcal{D}_T)$ of unlabeled samples from $\mathcal{X}$.

\begin{definition}[Domain Adaptation Learning model]
A concept class $\mathcal{C}$ over $\mathcal{X}$ is DA learnable if there exists an algorithm $\mathcal{A}$ such that for every concept $c\in \mathcal{C}$, for any distributions $\mathcal{D}_S$ and $\mathcal{D}_T$ on $\mathcal{X}$, and for all $\epsilon, \delta > 0$ if $\mathcal{A}$ is given access to $EX(c, \mathcal{D}_S)$, and $EX(\mathcal{D}_T)$ and inputs $\epsilon$ and $\delta$, then with probability at least $1-\delta$, $\mathcal{A}$ outputs a hypothesis $h \in \mathcal{C}$ satisfying $error_{\mathcal{D}_T}(h)\leq\epsilon$.
This probability is taken over random examples drawn by calls to $EX(\mathcal{D}_T)$.

If $\mathcal{A}$ runs in time polynomial in $1/\epsilon$, $1/\delta$, $n$ and $size(c)$ we say that $C$ is efficiently DA learnable.
\end{definition}

\noindent In light of the definition of the DA learning model, it appears useful to have a tool to measure the distance between distributions. The definition of distance that we employ is $L1$ distance as used in Bartlett et al.\cite{DBLP:conf/colt/Bartlett92}.
\begin{definition}[Distribution distance]
Given two distributions $\mathcal{D}_S$ and $\mathcal{D}_T$ on the universe $\mathcal{X}$ let their L1 distance $d(\mathcal{D}_S, \mathcal{D}_T)$ be:
    $$d(\mathcal{D}_S, \mathcal{D}_T) = \sup_{E\subseteq \mathcal{X}}\left|P_{\mathcal{D}_S}(E) -P_{ \mathcal{D}_T}(E)\right|$$
\end{definition}
\noindent Using the $L1$ metric gives the most tractable definition of distance between two distributions. It consists of taking the maximum change in probability over all events in the universe. 

In order to show our bounds for rejection sampling in section 4, we also need the definition of discrepancy between distributions  \cite{DBLP:conf/alt/MohriM12}. This is another approach to represent the distance between distributions, the main difference from $L1$ distance being that it is done in terms of a loss function. We define a loss function as a map $L : \mathcal{Y} \times \mathcal{Y} \rightarrow \mathds{R}_{+}$. In the rest of our analysis, we assume that $L$ is bounded by some constant $M > 0$. For any hypothesis $h : \mathcal{X} \rightarrow
\mathcal{Y}$ and any distribution $\mathcal{D}$ over the input space $\mathcal{X}$, we denote by $\mathcal{L}_{\mathcal{D}} (h)$ the expected loss of $h$:
\begin{equation}
\mathcal{L}_{\mathcal{D}}(h) = \mathbb{E}_{x\sim P_{\mathcal{D}}}[L(h(x), y)]
\end{equation}
Notice that the above reduces to $error_\mathcal{D}(h)$ the error probability of hypothesis $h$ under distribution $\mathcal{D}$ if we use the PAC loss function $L(h(x), y) = 1$ if $h(x)\neq y$, and 0 otherwise. 
\begin{definition}[Distribution discrepancy]
Given two distributions $\mathcal{D}_S$ and $\mathcal{D}_T$  on the universe $\mathcal{X}$ and a set $H$ of hypotheses, define the discrepancy $\mathrm{disc}(\mathcal{D}_S, \mathcal{D}_T)$ between the two distribution as:
$$
   \mathrm{disc}(\mathcal{D}_S, \mathcal{D}_T) = \sup_{h\in H}\left|\mathcal{L}_{\mathcal{D}_S}(h) - \mathcal{L}_{\mathcal{D}_T}(h)\right|
$$
\end{definition} 

\subsection{Assumptions}
As noted by Shai Ben-David \cite{Ben-DavidBCKPV10,Ben-DavidU12}, a basic analysis of the definitions shows that DA learning can be impossible when the source and target distributions have non-intersecting domains. Intuitively, training will not provide any useful information about the concept without further assumptions on the domain space and concept class. To avoid taking into account these scenarios, it is common to enforce a bounded ratio between the distributions at all points. However, this is an overly restrictive definition. We therefore follow Ben-David et al. \cite{Ben-DavidBCKPV10} in adopting a relaxation of this definition. The relaxation that they propose only requires the domain of the target distribution to be included in the domain of the source distribution, and sets a bound on the ratio of probabilities of events in the intersection of their domains. 
%Our statement is an improvement on their definition as the existence of the ratio implies that the target domain is included in the source domain, without requiring a separate condition. 

\begin{definition}[Weight ratio]
Given two distributions $\mathcal{D}_S$ and $\mathcal{D}_T$ over universe $\mathcal{X}$ the weight ratio $W(\mathcal{D}_S, \mathcal{D}_T)$ is defined as: 
% the smallest $w$ that satisfies the following expression for all $E\in\mathcal{X}$ where $P_{\mathcal{D}_S}(E) \neq 0$:
$$W(\mathcal{D}_S, \mathcal{D}_T) = \inf_{E\in Z}\frac{P_{\mathcal{D}_S}(E)}{P_{\mathcal{D}_T}(E)}$$
Where $Z = \{E : E\subseteq \mathcal{X} ~\text{and}~ P_{\mathcal{D}_T}(E) \neq 0\}$.
%$$w\min\left\{P_{\mathcal{D}_T}(E),P_{\mathcal{D}_S}(E)\right\} \geq  |P_{\mathcal{D}_T}(E)-P_{\mathcal{D}_S}(E)|$$
\end{definition}

%\noindent Note that this definition implies that $W(\mathcal{D}_S, \mathcal{D}_T)$ is always positive and greater than 1. 

\noindent Similar to \cite{Ben-DavidBCKPV10}, we will assume in our analysis of the DA problem that the domain of the target distribution is included in the domain of the source distribution and that there exists a bounded weight ratio such that for some $w \in \mathds{R}_+$: $$W(\mathcal{D}_S, \mathcal{D}_T) = 1/w$$.

\section{Can DA learning be efficient?}

In this section we present the main contributions of our work. We demonstrate that DA is not any harder than PAC learning. This naturally refutes a common misunderstanding of the conclusions of the work by Ben-David et al. \cite{Ben-DavidU12}. 
%We then derive general conditions on the hardness of the Left/Right problem introduced by Kelly et al. \cite{LeftRight} through a reduction to the Domain Adaptation learning problem.

\subsection{Recontextualization of Ben-David et al. \cite{Ben-DavidU12}}
The commonly cited Ben-David et al. paper suggests that Domain Adaptation is not efficiently PAC learnable. Via a recontextualization of their main result, we will show that the plight of Domain Adaptation is not as grim as has been reported. First, we will outline the main steps in their reasoning, starting with a formal presentation of the main result of their work:\\

\noindent \textbf{Ben-David et al. main result}
 \textit{
For every finite domain $\mathcal{X}$, for every $\epsilon, \delta > 0$, no algorithm can efficiently solve the DA problem even with $W(\mathcal{D}_S,\mathcal{D}_T) \geq 1/2$ assuming $s$ and $t$, the number of samples from the labeled and unlabeled oracles respectively, are such that $s+t < \sqrt{(1-2(\epsilon + \delta))|\mathcal{X}|}-2$.}\\

\noindent A close investigation of this result requires a brief exposition of the Left/Right problem introduced by Kelly et al. \cite{LeftRight}. 
Let $\mathcal{D}_L, \mathcal{D}_R$ and $\mathcal{D}_M$ be distributions over $\mathcal{X}$. Let  $L$ and $R$ be sets of  $l$ and $r$ independent draws from $\mathcal{D}_L$ and $\mathcal{D}_R$ respectively, and let $\mathcal{D}_L$ and $\mathcal{D}_R$ be the sets of points $E$ from $\mathcal{X}$ that $P_{\mathcal{D}_L} (E)$ and $P_{\mathcal{D}_R} (E)$ are non-zero.
Let $M$ be a set of  $m$ independent draws from $\mathcal{D}_M$ which is either $\mathcal{D}_L$ or $\mathcal{D}_R$. 
The goal of the Left/Right problem is to predict whether $M$ was generated according to $\mathcal{D}_L$ or $\mathcal{D}_T$.

\begin{definition}[Left/Right problem]
We say that a learning algorithm $\mathcal{A}$ efficiently solves the Left/Right problem if, given samples $L$, $R$ and $M$ of sizes $l,r$ and $m$, it outputs the correct answer with probability at least $1-\gamma$ in time polynomial in $1/\gamma, l, r$ and $m$. 
\end{definition}

\noindent Ben-David et al. prove their main result in two steps. They first show that the Left/Right problem is not efficiently solvable because there is a specific case of the problem that is not efficiently solvable. Then, they prove that the Left/Right problem reduces to the DA problem. The first step gives a bound on the sample size needed to solve the Left/Right problem by using the specific instance of $\mathcal{D}_R \cap \mathcal{D}_L = \emptyset$, $|\mathcal{D}_R| = |\mathcal{D}_L|$, and $\mathcal{D}_R \cup \mathcal{D}_L = \mathcal{X}$. Through reduction of this into the domain adaptation problem, they translate the bound obtained in the first step into the one stated in the main result of their paper. 

Our recontextualization posits that the implications of this result are not as wide ranging as the paper suggests. The formal statement of this bound on solving the DA problem seems to indicate that it holds for all instances of the DA problem. However, this is a misunderstanding of the reduction step in the proof. Rather than showing the hardness of the general DA problem even under strict assumptions, this result merely shows that there exists a sub class of the Domain Adaptation problem that is not efficiently PAC learnable. This is because the reduction from the Left/Right problem has no implications on the sub-class of the Domain Adaptation that is relevant: the one where the non-covariate shifted problem is efficiently PAC learnable. We can see this by first showing that the non-covariate shifted version of the specific Left/Right problem used in the proof above is also not efficiently learnable.

Take the left and right distributions to be the ones described in the reduction by Ben-David et al., where $\mathcal{D}_R \cap \mathcal{D}_L = \emptyset$, $|\mathcal{D}_R| = |\mathcal{D}_L|$, and $\mathcal{D}_R \cup \mathcal{D}_L = \mathcal{X}$. The source distribution $\mathcal{D}_S$ can be described as follows: a new data point is generated by flipping a fair coin and choosing the next point from $\mathcal{D}_L$ if it lands heads and from $\mathcal{D}_R$ otherwise. We label points generated from the source distribution using the PAC-loss function, with a value of 0 for heads and 1 for tails. This matches the source distribution from reduction in the Ben-David et al. paper which selects elements from $L\times \{0\} \cup R\times \{1\}$. The target distribution $\mathcal{D}_T$ is identical except without the labels. Thus, we have a non-covariate shifted version of the Left/Right problem instance presented in the original reduction. The question now is whether an algorithm can efficiently produce a hypothesis $h$ that is correct with probability of at least $1-\delta$ and $error_\mathcal{D} (h) \leq \epsilon$. We see that this is not possible; an adaptation of the proof presented in the the Ben-David et al. paper suffices, but we will present an alternative method here for readability:
\begin{proof}
For any given data point generated by the target distribution, we definitively know its label if it is already present in the training set. Otherwise, we have no other information about the label besides that it is equally likely that the label is either from the right or the left distributions. Thus we have an error:  $$\epsilon \geq \frac{1}{2}\left(\frac{n-1}{n}\right)^k + \left(1- \left(\frac{n-1}{n}\right)^k \right) $$
where $n = |R| + |L| = |\mathcal{X} |$ in this instance and $k$ is the training sample size. Simplifying this gives us a bound of approximately $ k \geq \frac{1}{\log\frac{n}{n-1}}\log(1-\delta)^2  \approx n\log(1-\delta)^2 $
and since $n$ scales exponentially with the dimension of the input space, this bound implies that the non-covariate shifted version of the Left/Right instance presented in the reduction is hard. 
\end{proof}

Thus, we have shown that there exists an instance of the standard PAC learning problem (with no covariate shift or domain adaptation) that is hard. If we extrapolate on the same basis as the Ben-David et al. paper, we would say that there exists no algorithm which solves PAC learning efficiently, namely that PAC learning is hard. But the statement that there are instances of standard PAC learning which are hard is trivial. The motivation behind the original PAC learning model was a desire to understand which problems are efficiently learnable and which are not, so it is not surprising that there exists some which are not efficiently learnable. 

It is a common technique to provide a lower bound on efficiently solving a class of problems by providing an instance of the problem that has some bounds on efficiency. The idea is that if an algorithm is said to solve a class of problems with a given complexity, then it has to solve any instance in that complexity or better. Thus, the presence of an instance that cannot efficiently be solved leads us to say that the class of problems is not efficiently solvable. This is certainly true in the case of domain adaptation and PAC learning as well, but it is important to examine whether the results are relevant. Just as the statement that PAC learning cannot be efficiently solved in general is a non-relevant and trivial result so too is the statement that DA is not efficiently PAC learnable. 

Just because an example of DA that is not efficiently PAC learnable, that does not mean that every instance of the DA problem is hard. In fact, there may be entire sub-classes that are efficiently learnable. We only care about the sub-class where the non-shifted problem is efficiently PAC learnable because if an instance of a non-shifted PAC problem is hard, then we should not expect the addition of a covariate shift to not be hard. This is precisely the case with the main result from the Ben-David et al. paper. Instead, we want to explore what happens if the original non-DA problem is efficiently PAC learnable. Is that still the case once we have have a covariate shift? This is explored below. 

\subsection{Domain Adaption on PAC Learnable Problems}
We show that the Domain Adaptation problem does not add any more complexity to the underlying PAC learning problem under the assumption that the source and target distributions have finite weight ratio $1/w$. PAC learning under DA involves a polynomial of $w$ increase in the number of samples required to learn.

\begin{theorem}
Let $\mathcal{C}$ be a concept class over input space $\mathcal{X}$ that is efficiently PAC learnable. It follows that $\mathcal{C}$ is also DA learnable for source and target distributions $\mathcal{D}_S$ and $\mathcal{D}_T$ over $\mathcal{X}$ satisfying $W(\mathcal{D}_S, \mathcal{D}_T) = 1/w$ in a number of steps that is polynomial in the relevant parameters and $w$.
\end{theorem}

\begin{proof}
First, we notice that the assumption $W(\mathcal{D}_S, \mathcal{D}_T) = 1/w$ implies that $\frac{\mathcal{D}_T(x)}{\mathcal{D}_S(x)} \leq w$ almost everywhere, as both $\mathcal{D}_S$ and $\mathcal{D}_T$ are measures over $\mathcal{X}$. This simply means that if for all events the probability ratio is bounded, then the distribution ration is also bounded by the same constant. \\
Second, let $L(h(x), c(x))$ be the PAC loss such that $L(h(x), c(x)) = 1$ when $h(x)\neq 1$ and 0 otherwise. We can express the error of hypothesis $h$ in terms of the PAC loss as:
\begin{equation}
error_{\mathcal{D}_S}(h) = \int_{\mathcal{X}} L(h(x), c(x)) \mathcal{D}_S(x) dx
\end{equation}
Furthermore, we can restrict the above to the domain $\mathcal{X}_D$ of $\mathcal{D}_S$, and use our assumption on the weight ratio $\frac{\mathcal{D}_T}{\mathcal{D}_S}$ to obtain the following inequality.
\begin{equation}
w ~ error_{\mathcal{D}_S}(h) \geq \int_{\mathcal{X}_D} L(h(x), c(x)) \mathcal{D}_S(x) \frac{\mathcal{D}_T(x)}{\mathcal{D}_S(x)} dx = error_{\mathcal{D}_T}(h)
\end{equation}
Where in the last step we use our assumption about that the domain of the target distribution $\mathcal{D}_T$ is included in the domain of the source distribution $\mathcal{D}_S$.
This shows that Domain Adaptation only deteriorates the error by a factor of $w$, and that under the above assumptions it can be PAC learned in a number of steps that is polynomial $w$ and the other relevant parameters.
\end{proof}

\section{Rejection Sampling for Discrete Distributions}
Theorem 1 shows that Domain Adaptation on a concept class that is PAC learnable increases the number of samples by a polynomial factor in the weight ratio of the source and target distribution. However, the degree of the polynomial depends on problem-specific assumptions. The degree could potentially be very large, in which case Domain Adaptation setting would increase significantly the number of training samples required for accurate predictions compared to the pure PAC problem. In literature, there are many examples of effective methods for specific problems that help solve a shift between training and testing distributions \cite{DBLP:conf/icml/BickelBS07,DBLP:conf/nips/HuangSGBS06,DBLP:journals/jip/TsuboiKHBS09}. Our work provides a framework to understand the reasons why these methods work, and their underlying assumptions. An important example of such empirical approach to solve the Domain Adaptation problem is rejection sampling \cite{DBLP:journals/sigpro/MartinoM10}. We provide an analysis of the rejection sampling algorithm that demonstrate that under certain scenarios it can achieve second degree polynomial increase in the number of samples.

We begin our analysis by relating the errors under the source and target distributions with the distance between the distributions. As mentioned before we use the $L1$ distance as the primary metric to express the distance between distributions. 
\begin{proposition} 
Assuming a loss function with given upper bound $M$ and some measure theoretic constraints on $\mathcal{D}_S,\mathcal{D}_T$:
\begin{equation}
    \mathrm{disc}(\mathcal{D}_S, \mathcal{D}_T) \leq 2 M d(\mathcal{D}_S, \mathcal{D}_T)
\end{equation}
\end{proposition}
\begin{proof}
See Appendix 1 for the proof. 
\end{proof}

\begin{proposition}
\begin{align}
\mathcal{L}_{\mathcal{D}_{2}}(h) &\leq \mathcal{L}_{\mathcal{D}_{S}}(h) +  \mathrm{disc}(\mathcal{D}_S, \mathcal{D}_T)\\
&\leq \mathcal{L}_{\mathcal{D}_{S}}(h) +  2Md(\mathcal{D}_S, \mathcal{D}_T)
\end{align}
\end{proposition}

\begin{proof}
The first line comes directly from the definition of discrepancy, and the second line comes from using Proposition 1. 
\end{proof}

Assuming the usual PAC learning loss function defined earlier with $M = 1$, we can rewrite equation (6) in terms of the error of a hypothesis $h$ under distributions $\mathcal{D}_S$ and $\mathcal{D}_T$:
\begin{equation}
error_{\mathcal{D}_{T}}(h) \leq error_{\mathcal{D}_{S}}(h) +  2d(\mathcal{D}_S, \mathcal{D}_T)\end{equation}

\subsection{Rescaling Distributions}
In this section we provide our main result for rejection sampling applied to the Domain Adaptation problem.

\begin{theorem}

Assume that there is a PAC learning algorithm $\mathcal{A}_S$ for concept class $\mathcal{C}$. Let $\mathcal{D}_S, \mathcal{D}_T$ be two discrete distribution over discrete universe $\mathcal{X}$. Assume that we are given access to an oracle $EX_S = EX(c, \mathcal{D}_S)$ which outputs labeled samples $\langle x, y = c(x)\rangle$ with $x\sim P_{\mathcal{D}_S}$, and to another oracle $EX_T = EX(\mathcal{D}_T)$ that outputs unlabeled samples $\langle x \rangle$ with $x \sim P_{\mathcal{D}_T}$. Further assume the following two properties of $\mathcal{D}_S$ and $\mathcal{D}_T$:
\begin{enumerate}
    \item The two distributions are discrete distributions with finite standard deviation less than some constant $s$.
    \item The two distributions $\mathcal{D}_S$ and $\mathcal{D}_T$ have weight ratio of $1/w = W(\mathcal{D}_S, \mathcal{D}_T)$.
\end{enumerate}
Under these assumptions there exists a PAC learning algorithm $\mathcal{A}_T$ that outputs hypothesis $h$ such that for any $\epsilon, \delta > 0$
$$error_{\mathcal{D}_T}(h) \leq \epsilon$$
with probability at least $1-\delta$ in a number of steps that is polynomial in $1/\epsilon, 1/\delta$ and $k$ and second degree polynomial in $w$.
\end{theorem}

\subsubsection{Description of the Algorithm}
As our only source of labeled data is distributed according to $\mathcal{D}_S$, our general approach will consist in using rejection sampling to create a labeled data set following a distribution that approximates $\mathcal{D}_T$ as closely as possible. We propose the following algorithm for rejection sampling:
\begin{enumerate}
    \item Using $m_1$ samples from oracles $EX_S$ and $EX_T$ obtain estimators $\hat{p}_{1,i},\hat{p}_{2,i}$ of the probabilities that random samples from $\mathcal{D}_S$ and $\mathcal{D}_T$ respectively equal $i \in \mathcal{X}$ (we use $i$ to denote discrete elements in the input space). 
    \item Create a new data set by taking $m_2$ samples from $\mathcal{D}_S$. When a new value $i$ is sampled, accept it with probability proportional to $\frac{\hat{p}_{2,i}}{\hat{p}_{1,i}}$.
    \item Train the existing PAC algorithm on the new data set
\end{enumerate}
Let $\mathcal{D}_f$ be the distribution obtained through the process described above. 

\subsubsection{Correctness of the Algorithm} We divide the analysis of the rejection sampling algorithm in several parts. We begin by proving the following claim showing that the result in Theorem 2 is true, under an assumption that will be demonstrated in the next paragraph.
\begin{claim}
Assume that using the proposed algorithm we construct $\mathcal{D}_f$ such that $d(\mathcal{D}_f,\mathcal{D}_T)\leq \epsilon/4$ with probability greater than $1-\delta/2$.
The hypothesis obtained by training $\mathcal{A}_S$ with parameters $\epsilon/2, \delta/2$ on $\mathcal{D}_f$ will, with probability at least $1-\delta$ give an error rate of at most $\epsilon$ when tested on $\mathcal{D}_T$.
\end{claim}

\begin{proof}
% \noindent{Proof of claim: }
%In the context of Proposition $1$, let the loss function be the $0-1$ PAC loss function. Plug in $\mathcal{D}_T$ and $\mathcal{D}_f$ as distributions, with $M=1$. Then
%$$\mathrm{disc}(\mathcal{D}_T,\mathcal{D}_f)=\underset{h\in H}{\sup}|error_{\mathcal{D}_T}(h)-error_{\mathcal{D}_f}(h)|$$
%By Proposition 1 and our assumption, we get that 
%$$\underset{h\in %H}{\sup}~|error_{\mathcal{D}_T}(h)-error_{\mathcal{D}_f}(h)|\leq \epsilon/2$$
Let $h$ be the hypothesis selected by the PAC algorithm $\mathcal{A}_S$ on $\mathcal{D}_f$. 
%Then the above still holds for $h$, with probability at least $1-\delta/2$:
%$$|error_{\mathcal{D}_T}(h)-error_{\mathcal{D}_f}(h)|\leq \epsilon/2$$
We know from the PAC guarantees of $\mathcal{A}_S$ that with probability $1-\delta/2$, $error_{\mathcal{D}_f}(h)<\epsilon/2$. Hence, using Proposition 2 and the union bound, with probability at least $1-\delta$,
$error_{\mathcal{D}_T}(f)\leq \epsilon$
as desired. This concludes the proof of the claim.
\end{proof}

To complete the proof of Theorem 2, we are thus left with showing that using the procedure described above we can efficiently approximate $\mathcal{D}_T$ by $\mathcal{D}_f$ with high probability.

\subsubsection{Bounding the Distance between Distributions}
Our approximation of $\mathcal{D}_T$ by $\mathcal{D}_f$ has the following source of error:  we do not know the true values of $p_{1i},p_{2i}$ and instead we are rejecting using estimates $\hat{p}_{1i},\hat{p}_{2i}$. To solve this issue we need to sample enough points in step $1$ of the algorithm so that the estimates $\hat{p}_{1i},\hat{p}_{2i}$ are close to true values $p_{1i},p_{2i}$.

\begin{lemma}[Finite Case]
    Let $\mathcal{D}_S$ and $\mathcal{D}_T$ be the source and target distributions as defined above, with the additional constraint that both distributions have a finite support set. Assume that set to be $\{1, 2, ...,n\}$. Then for any $\epsilon, \delta$ we show that with probability at least $1-\delta$, $d(\mathcal{D}_f,\mathcal{D}_T)<\epsilon$ by taking $m_1$ samples in step $1$ of our algorithm to approximate $p_{1i},p_{2i}$, where $m_1$ is polynomial in $n, \frac{1}{\epsilon},\frac{1}{\delta}$, and second degree polynomial in $w$.
\end{lemma}

\noindent The full proof of Lemma 1 can be found in Appendix $2$. We prove Lemma 1 using the Chernoff bound for points with probability mass $p_{1i}$ larger than some threshold while ignoring points with negligible probability mass.

\subsubsection{Reduction to Finite Case} 
We show that we can restrict the domain of $\mathcal{D}_S$ and $\mathcal{D}_T$ to a finite set that has more than $1-\epsilon/2$ of the probability mass under both distributions. 
From Chebyshev's inequality, the probability that a sample is $k$ away from the mean is bounded above by $\frac{s^2}{k^2}$ for both of our distributions, where $s$ is the upper bound on the standard deviations. We want:  
$$\frac{s^2}{k^2}<\epsilon/2$$ 
\noindent This means that it is enough to pick $n=2s\sqrt{2/\epsilon}$ as we have values on the both sides of the mean. We can now use Lemma 1 above, with the specified $n$, and plugging in $\epsilon/2$ instead of $\epsilon$ for the desired accuracy parameter to complete our demonstration of Theorem 2. Notice that the part of the domain that was not in the finite restricted domain set $\{1,...,n\}$ can only contribute $\epsilon/2$ to the  overall distance. 

\subsubsection{Algorithm Complexity}
We show that the complexity of each step of the algorithm is polynomial in the relevant parameters, and at most second degree polynomial in the weight ration $w$.

\noindent \textbf{Step 1:} the number $m_1$ of samples that we need is, from Lemma 1, a polynomial in  $n,w,\frac{1}{\epsilon},\frac{1}{\delta}$.

\noindent\textbf{Step 2:} To analyze the complexity of step 2, first let $m_2'$ be the required number of examples by the given algorithm $\mathcal{A}_S$ to provide $\epsilon/2, \delta/2$ PAC guarantees on $\mathcal{D}_f$. This is polynomial in the relevant parameters. Furthermore, condition 2 in the setup of Theorem $1$ gives us the bound:
$$\frac{1}{w}\leq \frac{P_{\mathcal{D}_S}(i)}{P_{\mathcal{D}_T}(i)}$$
for any value $i$. After re-normalization we can bound below the rejection probabilities by $\frac{1}{w^2}$ for any $i$. By taking  $m_2=m_2'w^2\log (4/\delta)$ samples, even with a rejection probability $\frac{1}{w^2}$ we still get at least $m_2'$ samples after rejection with probability at least $1-\delta/4$ by Chernoff's Inequality. As $m_2'$ is polynomial in the relevant parameters by the existing PAC guarantees of the $\mathcal{A}_S$, $m_2$ will also be polynomial in said parameters, hence the complexity of step  2 is also polynomial in the relevant parameters. 

\noindent \textbf{Step 3:} Last but not least, complexity of step 3 is polynomial in $1/\epsilon,1/\delta$ and other relevant parameters coming with the algorithm provided, due to the PAC guarantee.\\

Overall, the complexity of all the rejection sampling algorithm described is:
\begin{equation}
f(1/\epsilon,1/\delta)w^2\log (4/\delta)+\log(8s\sqrt{2/\epsilon})+\log(1/\delta))\bigg(\frac{1}{\epsilon^3}2^{15}s\sqrt{2/\epsilon}w^2\bigg)
\end{equation}
\noindent where $f$ is the polynomial coming from the PAC guarantee for the given algorithm.

\section{Conclusion}
In this paper, we start out in section 3 by re-framing a previous result in literature suggesting that learning in the context of Domain Adaptation is hard. Despite the existence of some hard Domain Adaptation problems shown in the literature, we argue that this does not condemn Domain Adaptation as a whole. In Theorem 1, we prove that if a problem is efficiently PAC learnable, then the introduction of a bounded covariate shift, with finite density ratio between distributions, does not add much complexity to the underlying PAC learning problem, which remains computable in polynomial number of samples.

On the surface, this would suggest that the remedy in cases of covariate shift is simply to collect more data as efficient PAC learning implies that the error scales polynomially with the number of training data points required. However, the density ratios between the two distributions combined with the possibility of a high-order polynomial dependency suggest that a covariate shift could easily require orders of magnitude more data than in the non-shift case. Lowering the data burden even further would then seem to be the theoretical motivation for the various techniques and algorithms developed for combating covariate shift. In particular, we provide an analysis of the rejection sampling algorithm showing that it can limit the data burden to a second degree polynomial in the density ratio under the assumption of discrete distributions with finite standard deviation.

An area of future research would be to pursue a general lower bound for the added data burden in the covariate shift setting that is agnostic to the algorithm used. This would be in contrast to our Theorem 1 which in some sense provides an upper bound on the covariate shift data burden and add to our results in section 4 by generalizing the bound we achieved for the discrete and finite variance case.

\section*{Acknowledgement} 
We are grateful to Professor Leslie Vliant and Daniel Moroz for inspiring us to work on this problem. We thank Marco Sangiovanni and Vinitra Swamy for the rigorous review of our results.

\newpage
\section*{Appendix 1:}
\begin{proof}[Proposition 1]
To demonstrate Proposition 1 it is sufficient to show that for all hypothesis $h$:
$$\left|\mathcal{L}_{\mathcal{D}_S}(h) - \mathcal{L}_{\mathcal{D}_T}(h)\right| \leq 2 M \sup_{E\subseteq X} |P_{\mathcal{D}_S}(E) - P_{\mathcal{D}_T}(E)|$$
In particular, let $f = L(h(x), y)$, we then have that for any two measures $\mathcal{D}_S, \mathcal{D}_T$:
$$\left|\mathcal{L}_{\mathcal{D}_S}(h) - \mathcal{L}_{\mathcal{D}_T}(h)\right| = \left|\int_\mathcal{X} f\mathcal{D}_S(x)dx - \int_X f\mathcal{D}_T(x)dx\right|$$
Futhermore, from the Radon-Nikodym theorem we know that $\exists g$ such at $g$ is measurable and the above is equal to:
$$=\left|\int_\mathcal{X} f (1-g)\mathcal{D}_S(x)dx\right|$$
Now Let $E_{+} = \{ x\in \mathcal{X} : g(x) < 1\}$, and $E_{-} = \{ x\in \mathcal{X} : g(x) \geq 1\}$. We break the above integral in two and use the triangle inequality:
$$=\left|\int_{E_+} f (1-g)\mathcal{D}_S(x)dx + \int_{E_-} f (1-g)\mathcal{D}_S(x)dx\right|$$
$$\leq\left|\int_{E_+} f (1-g)\mathcal{D}_S(x)dx\right| + \left|\int_{E_-} f (1-g)\mathcal{D}_S(x)dx\right|$$
Now we can use the fact that $f(1-g)$ is always positive on $E_+$, while $f(1-g)$ is always negative on $E_-$, and the fact that $f$ is bounded by $M$. The above then becomes:
$$\leq M(P_{\mathcal{D}_S}(E_+) - P_{\mathcal{D}_T}(E_+)) + M \left|P_{\mathcal{D}_S}(E_-) - P_{\mathcal{D}_T}(E_-)\right|$$
$$\leq 2M \sup_{E\subseteq \mathcal{X}}\left| P_{\mathcal{D}_S}(E) - P_{\mathcal{D}_T}(E)\right|$$
\end{proof}

\section*{Appendix 2:}

%\noindent{\bf Proof of Lemma 1. }By definition, 
\begin{proof}[Lemma 1]
$$d(\mathcal{D}_T,\mathcal{D}_f)=\sup_E|P_{\mathcal{D}_T}(E)-P_{\mathcal{D}_f}(E)|$$
We can rewrite 
$$d(\mathcal{D}_T,\mathcal{D}_f)=\sup_E\bigg|\sum_{i\in E}  P_{\mathcal{D}_T}(i)-P_{\mathcal{D}_f}(i)\bigg|$$
Thus,
$$d(\mathcal{D}_T,\mathcal{D}_f)\leq \sup_E\sum_{i\in E}  |P_{\mathcal{D}_T}(i)-P_{\mathcal{D}_f}(i)|$$
so 
$$d(\mathcal{D}_T,\mathcal{D}_f)\leq \sum_{i=1}^n  |P_{\mathcal{D}_T}(i)-P_{\mathcal{D}_f}(i)|$$
Now we denoted $P_{\mathcal{D}_T}(i)$ with $p_{2,i}$, and from the definition of our algorithm, $P_{\mathcal{D}_f}(i)=|p_{2i}-\hat{p}_{2i} \frac{p_{1i}}{\hat{p}_{1i}}|$. Hence,
$$d(\mathcal{D}_T,\mathcal{D}_f)\leq \bigg(\sum_i \bigg|p_{2i}-\hat{p}_{2i} \frac{p_{1i}}{\hat{p}_{1i}}\bigg|\bigg)$$
Let's now divide our points in two subsets: a point is heavy if $p_{1i}\geq \frac{\epsilon}{2nw}$ and small otherwise. 
Consider now $i$ a heavy point. Remember that $\hat{p}_{1i}$ is an estimator of ${p}_{1i}$ obtained from $m_1$ samples.  From Chernoff's inequality we have
$$P(|{\hat{p}_{1i}}-p_{1i}|>p_{1i}\epsilon/16)\leq 2e^{-m_1\epsilon^2p_{1i}/256}$$
We want this to be less than $\delta/2n$. By taking log and simplifying, we want
$$m_1\geq (\log4n)+\log(1/\delta))\bigg(\frac{1}{\epsilon^2}256\frac{1}{p_{1i}}\bigg)$$
But ${p_{1i}}> \frac{\epsilon}{2nw}$ as $i$ is a heavy point so it is enough to pick 
$$m_1\geq (\log(4n)+\log(1/\delta))\bigg(\frac{1}{\epsilon^3}2^{10}nw\bigg)$$
With such choice of $m_1$, we know that 
$$P(|{\hat{p}_{1i}}-p_{1i}|>p_{1i}\epsilon/16)\leq \delta/4n$$
We now want to bound $|p_{2i}-\hat{p_{2i}}|$. Again, using Chernoff, 
$$P(|p_{2i}-\hat{p_{2i}}|>p_{2i}\epsilon/16)\leq 2e^{-m_1\epsilon^2p_{2i}/256}$$
Using condition 2 in the setup of Theorem 1 we get that $\frac{p_{1i}}{p_{2_i}}\geq \frac{1}{w}$. Together with the fact that $i$ is heavy, we obtain that $p_{2i}\geq \frac{\epsilon}{2nw^2}$
Using the same argument as for $p_{1i}$above, it is clear that by choosing
$$m_1\geq (\log(4n)+\log(1/\delta))\bigg(\frac{1}{\epsilon^3}2^{11}nw^2\bigg)$$ we obtain
$$P(|p_{2i}-\hat{p_{2i}}|>p_{2i}\epsilon/16)\leq \delta/4n$$
Using the union bound, this means 
$$P\bigg(|p_{2i}-\hat{p_{2i}}|>p_{2i}\epsilon/16\bigcap |\frac{p_{1i}}{\hat{p}_{1i}}-1|>\epsilon/16\bigg)\leq \delta/2n$$
Remember that we are trying to bound 
$$\bigg|p_{2i}-\hat{p}_{2i} \frac{p_{1i}}{\hat{p}_{1i}}\bigg|$$
 Given that we effectively bounded $\frac{p_{1i}}{\hat{p}_{1i}}$ and $\big|p_{2i}-\hat{p}_{2i}\big|$ the maximum of the difference can be obtained when either $\hat{p}_{2i}$ is bigger than ${p}_{2i}$ and $\frac{p_{1i}}{\hat{p}_{1i}}$ greater than $1$ or $\hat{p}_{2i}$ is smaller than ${p}_{2i}$ and $\frac{p_{1i}}{\hat{p}_{1i}}$ smaller than $1$. When both our larger, the difference is at most $p_{2i}(1+\epsilon/16)^2- p_{2i}=p_{2i}(\epsilon/16)(2+\epsilon/16)<p_{2i}\epsilon/4$. The analysis in the other case is similar, giving the same bound. Thus we obtain
$$P\bigg(\bigg|p_{2i}-\hat{p}_{2i} \frac{p_{1i}}{\hat{p}_{1i}}\bigg|>p_{2i}\epsilon/4\bigg)\leq \delta/2n$$
for all heavy $i$. Note that $\sum_i p_{2i}\leq 1$. As we have at most $n$ heavy points,
$$P\bigg(\bigg(\sum_{{i\text{ heavy}}} \bigg|p_{2i}-\hat{p}_{2i} \frac{p_{1i}}{\hat{p}_{1i}}\bigg|\bigg)>\epsilon/4\bigg)<\delta/2$$
Hence
$$P\bigg(\sum_{{i\text{ heavy}}} \big|P_{\mathcal{D}_T}(i)-P_{\mathcal{D}_f}(i)\big|>\epsilon/4\bigg)<\delta/2$$
In particular, the above also gives us 
$$P\bigg( \big|\sum_{{i\text{ heavy}}}P_{\mathcal{D}_T}(i)-\sum_{{i\text{ heavy}}}P_{\mathcal{D}_f}(i)\big|>\epsilon/4\bigg)<\delta/2$$
which is equivalent to
$$P\bigg(\big|\sum_{{i\text{ light}}}P_{\mathcal{D}_T}(i)-\sum_{{i\text{ light}}}P_{\mathcal{D}_f}(i)\big|>\epsilon/4\bigg)<\delta/2$$
as sum of probabilities for all $i$ sum to 1.
But using the fact that from condition 2 $p_{2i}\leq wp_{1i}$ $$\sum_{{i\text{ light}}}P_{\mathcal{D}_T}(i)=\sum_{{i\text{ light}}}p_{2i}\leq \sum_{{i\text{ light}}}w p_{1i}$$
Now when $i$ light, $p_{1i}\leq \frac{\epsilon}{2nw}$ and furthermore we have at most $n$ light $i$
$$\sum_{{i\text{ light}}}P_{\mathcal{D}_T}(i)\leq \epsilon/4$$
so 
$$P\bigg(\sum_{{i\text{ light}}}P_{\mathcal{D}_f}(i)> \epsilon/4\bigg)<\delta/2$$
Now note that for $i$ small, $P_{\mathcal{D}_f}(i), P_{\mathcal{D}_T}(i)$ are sequences of positive numbers with sum bounded by $\epsilon/2$, so the expression 
$$\sum_{{i\text{ light}}} \big|P_{\mathcal{D}_T}(i)-P_{\mathcal{D}_f}(i)\big|$$
is maximized when the sequences look like $a_1,0,a_2,0,...$ and $0,b_1,0,b_2,...$, with a maximum of $\sum a_i+\sum i$, hence bounded by $2\cdot\epsilon/4=\epsilon/2$. Hence, with probability at least $1-\delta/2$,
$$\sum_{{i\text{ light}}} \big|P_{\mathcal{D}_T}(i)-P_{\mathcal{D}_f}(i)\big|\leq \epsilon/2 $$
and also 
$$\sum_{i\text{ heavy}} \big|P_{\mathcal{D}_T}(i)-P_{\mathcal{D}_f}(i)\big|\leq \epsilon/4 $$
with probability at least $1-\delta/2$ hence by union bound with probability at least $1-\delta$
$$\sum_i \big|P_{\mathcal{D}_T}(i)-P_{\mathcal{D}_f}(i)\big|\leq \epsilon$$
as desired.
\end{proof}

\end{document}